\documentclass{article}
\usepackage{natbib}
\usepackage{float}
\usepackage{rotating}
\usepackage{mathtools}
\usepackage{amsmath}

\usepackage{amsthm}
\usepackage{amsfonts}
\usepackage{tikz}
\usepackage{subfigure}
\usepackage{tabularx}
\usepackage[colorlinks=true]{hyperref}
\usetikzlibrary{arrows, fit}

\newtheorem{lemma}{Lemma}
\newtheorem{definition}{Definition}

\DeclareMathOperator*{\argmax}{argmax}
\DeclareMathOperator*{\argmin}{argmin}

\newcommand{\MaxCE}{MaxCE}

\renewcommand{\[}{\big[}

\newcommand{\kld}[2]{\mathrm{D_{KL}}\left(#1 \,\|\, #2\right)}
\renewcommand*{\eqref}[1]{Eq.~(\ref{#1})}

\newcommand*{\figref}[1]{Fig.~\ref{#1}}

\newcommand*{\secref}[1]{Sec.~\ref{#1}}

\allowdisplaybreaks 

\newcommand*{\expectcustom}[4][]{#1#2[\hspace{#3}#2[#4#2]\hspace{#3}#2]}

\newcommand*{\expectBig}[2][]{\expectcustom[#1]{\Big}{-0.32em}{#2}}

\newcommand*{\eexpectBig}[1]{\expectBig[\mathbb{E}]{#1}}

\newcommand*{\Dklbig}[2]{\mathrm{D_{KL}}\big[#1 \,\|\, #2\big]}

\newcommand*{\HHbig}[1]{\mathrm{H}\big[#1\big]}

\begin{document}

\title{The Advantage of Cross Entropy over Entropy in Iterative Information
Gathering}
\author{Johannes Kulick \and Robert Lieck \and Marc Toussaint}

\maketitle

\begin{abstract}
  Gathering the most information by picking the least amount of data is a common
  task in experimental design or when exploring an unknown environment in
  reinforcement learning and robotics. A widely used measure for quantifying the
  information contained in some distribution of interest is its
  entropy. Greedily minimizing the expected entropy is therefore a standard
  method for choosing samples in order to gain strong beliefs about the
  underlying random variables. We show that this approach is prone to temporally
  getting stuck in local optima corresponding to wrongly biased beliefs. We
  suggest instead maximizing the expected cross entropy between old and new
  belief, which aims at challenging refutable beliefs and thereby avoids these
  local optima. We show that both criteria are closely related and that their
  difference can be traced back to the asymmetry of the Kullback-Leibler
  divergence. In illustrative examples as well as simulated and real-world
  experiments we demonstrate the advantage of cross entropy over simple entropy
  for practical applications.
\end{abstract}

Information gain $\cdot$ Experimental design $\cdot$ Exploration
$\cdot$ Active learning $\cdot$ Cross entropy $\cdot$ Robotics


\section{Introduction}

When gathering information, agents need to decide where to sample new data.  For
instance, an agent may want to know the latent parameters of a model for making
predictions or it has a number of possible hypotheses and wants to know which
one is true. If acquiring data is expensive, as it is the case in real-world
environments or if a human expert answers queries, it is desirable to use the
least amount of data for gathering the most information possible. An agent
therefore should choose queries most informative for its learning
progress---which is referred to as \textit{active learning} and
\textit{experimental design}.

The commonly addressed task in these areas is to reduce the \textit{predictive
  uncertainty}, that is, to choose queries as efficiently as possible with the
aim of reducing the prediction errors of the model. In this paper, however, we
also consider a slighly different task, namely to reduce the uncertainty over
some \textit{(hyper) parameter} of the model which is not observable. In the
generative model \figref{fig:generative-model}, reduction of predictive
uncertainty aims at learning the function $f$, whereas the alternative task is
to learn about the parameter $\theta$.

A widely used approach for minimizing uncertainty over the hyper parameters is
to greedily minimize the expected entropy of the posterior distribution
$p(\theta|x,y,D)$. However, we can show that this greedy method can get trapped
in \emph{erroneous} low-entropy beliefs over $\theta$. We therefore suggest an
alternative measure: maximizing the expected cross entropy between the prior
$p(\theta|D)$ and the posterior $p(\theta|x,y,D)$. Although also being a
one-step criterion, our cross entropy criterion avoids local optima in cases
where the standard entropy criterion gets trapped. We demonstrate superior
convergence rates in empirical evaluations. We show that the difference between
the two criteria can be traced back to the asymmetry of the Kullback-Leibler
divergence (KL-divergence) and discuss this in detail. Furthermore, we show that
even in the standard case of reducing predictive uncertainty our criterion can
be used to improve the convergence rate by combining it with standard
uncertainty sampling.

In general, computing optimal solutions to experimental design problems requires
taking all possible future queries into account. This translates to solving a
partially observable Markov decision problem (POMDP) \citep{chong2009partially},
which generally is unfeasible to compute (see
e.g. \cite{kaelbling-et-al:98-aij}). In the special case of submodular objective
functions greedy one-step optimization has bounded regret with respect to the
optimal solution \citep{nemhauser1978analysis}. However, we show that the
standard expected entropy criterion for the $\theta$-belief is \emph{not}
submodular. Therefore, the naive greedy criterion of reducing $\theta$-belief
entropy is not guaranteed to have bounded regret, which is consistent with out
empirical finding that it can get trapped in erroneous low-entropy belief
states. Our cross entropy criterion, which measures \emph{change} in belief
space rather than entropy reduction, is less prone to getting trapped.

In the remainder of this paper we will first discuss related work. We then
formally introduce our method \emph{\MaxCE{}} and discuss the difference to the
standard approach of minimizing expected entropy. After this we draw the
connection to active learning methods for reducing predictive uncertainty. We
empirically evaluate our method in a synthetic regression and classification
scenario as well as a high-dimensional real-world scenario comparing it to
various common measures. We then show how \MaxCE{} can be used
in a robotic exploration task to guide actions. Finally we discuss the results
and give an outlook to future work.

\tikzstyle{plain_edge} = [->,shorten >=.1pt, >=stealth]
\tikzstyle{dashed_edge} = [dash pattern=on 3pt off 1pt            ,->,shorten >=.1pt, >=stealth]
\tikzstyle{dotted_edge} = [dash pattern=on \pgflinewidth off 2.5pt,->,shorten >=.1pt, >=stealth]
\tikzstyle{node} = [minimum size=24pt]

\begin{figure}[t]
  \centering
  \begin{tikzpicture}
    \node (y_stern)    at (-1,-2) [node,shape=circle,draw] {$y$};
    \node (x_stern)    at (-2,-1) [node,shape=circle,draw] {$x$}
    edge [plain_edge] (y_stern);
    \node (D)     at (1, -2) [node,shape=circle, draw] {$D$}
    edge [dashed_edge] (y_stern);
    \node (f)     at (0, -1.3) [node,dash pattern=on \pgflinewidth off 2.5pt,shape=circle,draw] {$f$}
    edge [dotted_edge] (D)
    edge [dotted_edge] (y_stern);
    \node (theta)     at (0, 0) [node,shape=circle,draw] {$\theta$}
    edge [dashed_edge] (y_stern)
    edge [dashed_edge] (D)
    edge [dotted_edge] (f);

  \end{tikzpicture}
  \caption{Generative model of the data as Bayesian network. The already known
    data $D$ as well as the new query $x$ and label $y$ are drawn from a
    distribution $f$ that is determined by (hyper) parameters $\theta$. Dashed
    arcs describe the dependencies after marginalizing out $f$, dotted arcs
    describe the dependencies before that.}
  \label{fig:generative-model}
\end{figure}
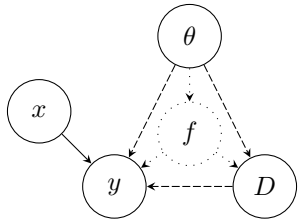

\section{Related Work}
\label{sec:rel_work}

Our method is closely related to \emph{Bayesian experimental design}, where
Bayes\-ian techniques are used to optimally design a series of experiments. The
field was coined by \cite{lindley_measure_1956}.
\cite{chaloner-verdinelli:95-jss} give an overview of the method and its various utility
functions. An experiment, possibly consisting
of several measurements, in this context can be seen as a single sample taking
in some parameter space. The classic utility function is to maximize the
expected Shannon information \citep{shannon-48} of the posterior over a latent
variable of interest, which corresponds to maximizing the expected neg.~entropy
of the posterior or equivalently the expected KL-divergence from posterior to
prior (see \secref{sec:expected_entropy_versus_cross-entropy}). Our \MaxCE{}
method is closely related in that maximizing the expected cross entropy also
corresponds to maximizing the expected KL-divergence but \emph{from prior to
  posterior}, that is, in the opposite direction. As a consequence, while the
traditional experimental design objective is not well suited for greedy
iterative optimization as it may get stuck in local optima (see
\secref{sec:min-entropy}) our \MaxCE{} criterion overcomes this flaw while
retaining the desired property of converging to low-entropy posteriors. Bayesian
experimental design recently has regained interest due to an efficient
implementation, the Bayesian Active Learning by Disagreement (BALD) algorithm
\citep{houlsby-et-al:arxiv-2011}, which exploits the equivalence of the
experimental design criterion to a mutual information in order to make the
computations tractable.

The general approach to minimize the number of queries for a learning task is
often called \textit{active learning}. As a general framework, active learning
comprises a variety of methods (see \cite{settles:12-book} for an overview) and
is successfully used in different fields of machine learning and for a wide
range of problems, as shown in a survey of projects using active learning
\citep{tomanek-olsson:09-alnlp}. However, it mainly focuses on reducing
predictive uncertainty (predictive error or predictive entropy) for a single
model, whereas our method aims at learning hyper parameters such as selecting
the correct model out of several possible candidates.

Model selection techniques on the other side mainly focus on criteria to
estimate the best model (or hypothesis) given a set of training data, which can
also be seen as learning of latent parameters of a model. Well known criteria
are Akaikes Information Criterion (AIC)
\citep{akaike:74-ac,burnham-anderson:04-smr} or the Bayesian Information
Criterion (BIC) \citep{schwarz:78-as,bhat-kumar:10-tr}. Both are based on the
likelihood ratios of models and are approximations of the distribution over the
latent variable.  Specifying concrete likelihood models we can infer the
distribution of the variable of interest directly and do not need to approximate
them. In certain cases it might however be useful to apply approximations to
speed up the method.  Another approach to rate a model is cross-validation (see
e.g.~\cite{kohavi:95-ijcai}), which statistically tests models with subsets of
the training data for their generalization error.  All model selection
techniques have in common that they measure the quality of a model given a data
set. They are not often used for actively sampling queries and in the case of
predictive error, this might actually fail for the ``Active Learning with Model
Selection Dilemma'' \citep{sugiyama-rubens:08-icdm}. We observe a similar
problem, when measuring predictive error in our experiments (see
\secref{sec:experiments}). Our methods on the other side is developed for
actively choosing queries.

\emph{Query-by-committee (QBC)} as introduced by \cite{seung-et-al:92-colt}
tries to use active learning methods for
\emph{version space} reduction. The version space is the space of competing
hypothesis. QBC finds new samples by evaluating
the disagreement within a set of committee members (that is, different
hypotheses)
concerning their predictions. These samples are then used to train the
different models. In a binary classification scenario disagreement is
easy to determine. In multi-class or regression scenarios it is harder
to define. One approach, as suggested by
\cite{mccallum-nigam:98-icml}, is to use as measure of disagreement
the sum of KL-divergences from each committee member's
predictive belief to the mean of all committee member's
predictive belief.
While QBS amis at
finding the correct hypothesis, it still focuses on the prediction error. We
will empirically compare our approach to QBC.

Another variant of active learning are expected model change methods such as the
expected gradient length algorithm \citep{settles-et-al:08-nips}. These methods
measure the change a model undergoes by adding another observation. Our method
is in spirit related and might arguably be classified as a variant of expected
model change since we are also interested in finding samples that contain a
maximum amount of information with respect to the model. However, we apply the
idea of greatest model change directly to the \emph{distribution} of hypotheses
by measuring the KL-divergence between the distribution before and after new
observations have been incorporated. In contrast, existing methods stay within
one fixed model and measure the change of this fixed model. Those methods can
thus not directly be used for discriminating between hypotheses.

In our work on joint dependency structure exploration
\citep{kulick-et-al:15-icra}, we used our method. There we focused on modeling the
joint dependency structure. We analyze these experiments with respect to the
MaxCE method further in Sec.~\ref{sec:robot-experiment}.

Similar to our robot experiment is the work of \cite{hausman2015active}.  They
state that the KL divergence is the information gain about a distribution, but
turn it around without further explanation and analysis. In this way, they
implemented our MaxCE criterion, as we will show. Their results support our
finding that MaxCE is an improvement above traditional Bayesian experimental
design.

For some experiments we use Gaussian Processes for regression and
classification.  See \cite{rasmussen-williams:06-book} for an
extensive introduction.

\section{Information Gathering Process}

Let $\theta$, $x$, $y$, $D$, and $f$ be random variables. $\theta$ denotes a
latent random variable indicating the hypothesis of interest, e.g. model class,
hyper parameter or other latent parameter. Conditional to $\theta$ we assume a
distribution over functions $f$, for instance a Gaussian Process prior,
generating the observable data. (In the
classification case these are discriminative functions.) $D$ are the data
observed so far consisting of $(x_i,y_i)$ input-output pairs where
$P(y_i|x_i,f)$ depends on $f$. $x$ is the input that is to be chosen actively
and $y$ is the corresponding output received in response. The graphical model in
\figref{fig:generative-model}, neglecting the dashed arcs, describes their
dependence. For gathering information over $\theta$ we will have to express the expected
information gain about $\theta$ depending on $x$ and usually eliminate $f$. In
the graphical model, after eliminating $f$, the dashed arcs describe the
dependencies.

The task now is to gather the most information with the least queries.
If we assume a ground truth distribution $P(\theta^*)$, we can formulate this
task for a given horizon of $K$ queries as to
minimize the KL-divergence between $P(\theta^*)$ and the posterior over $\theta$.

\begin{align}
  (x_1^*, \dots, x_K^*) = \argmin_{(x_1, \dots, x_K)} \kld{P(\theta^*)}{P(\theta
  | D)} 
\end{align}
with $D = \{(x_1, y_1), \dots, (x_K, y_K)\}$.

But unfortunately the distribution $P(\theta^*)$ is unknown and is in fact the
desired piece of information we want to infer. Thus we can immanently not
compute this KL-divergence. In many cases it is however reasonable to assume
$P(\theta^*)$ having a low entropy. This assumption holds e.g.~for the case of a
single ``true'' hypothesis or a single ``true'' value of a measured constant. Under
these circumstances it is reasonable to minimize the entropy of $P(\theta|D)$

\begin{align}
  (x_1^*, \dots, x_K^*) = \argmin_{(x_1, \dots, x_K)} H\left[P(\theta |
  D)\right] 
  \label{eq:task}
\end{align}
with $D = \{(x_1, y_1), \dots, (x_K, y_K)\}$.

This still includes reasoning over all $K$ future queries and is generally
computationally intractable. Thus the scenario we describe here will be
iterative.  The choice of a new query point $x$ is guided by an objective
function that scores all candidate points. The candidate with an optimal
objective value is then used to generate a new data point $(x,y)$ which, for the
next iteration, is added to $D$.

\subsection{Expected Entropy versus Cross-Entropy}
\label{sec:expected_entropy_versus_cross-entropy}

From \eqref{eq:task} and the general task to gather information it is very
intuitive to minimize expected hypotheses entropy\footnote{Note that this is
often called
the \emph{conditional entropy} and written $H(\theta|y)$. To avoid confusion, we
will not use this shorthand notation but explicitly state the
distribution we take the entropy of and over what random variable we will take
the expectation over. Also see Sec.~\ref{sec:submodular}.} in each step. This is a common utility function for Bayesian
experimental design \citep{chaloner-verdinelli:95-jss}:
\begin{align}
x_{NE}
&= \argmax_x \int_y - p(y|x,D) H[p(\theta|D,x,y)] ~. \label{eq:neg_entropy}
\end{align}
It is very instructive to rewrite this same criterion in various ways. We can
for instance subtract $H[p(\theta|D)]$, as it is a constant offset to the
maximizing operator (see App.~\ref{sec:dkl-transformations} for the detailed
transformations):
\begin{align}
  &\hphantom{=~} \argmax_x \int_y - p(y|x,D) H[p(\theta|D,x,y)] \\
 &= \argmax_x - \int_y p(y|x,D)~ H[p(\theta|y,x,D)] - H[p(\theta|D)] \label{eq:exp_entropy} \\
 &= \argmax_x \int_y p(y|x,D)~ \kld{p(\theta|y,x,D)}{p(\theta|D)} ~. \label{eq:expected_KLD}
\end{align}
These rewritings of expected entropy establish the direct
relation to Eq.~(3) and (4) in \cite{chaloner-verdinelli:95-jss}. We
find that $x_{NE}$ can be interpreted both as maximizing the expected neg.~entropy,
as in~\eqref{eq:neg_entropy}, or maximizing the expected KL divergence,
as in~\eqref{eq:expected_KLD}.

Minimizing the expected model entropy is surely one way of maximizing
information gain about $\theta$. However, in our iterative setup we
empirically show that this criterion can get stuck in local optima:
Depending on the stochastic sample $D$, the hypotheses posterior $p(\theta|D)$
may be ``mislead'', that is, having low entropy while giving high
probability to a wrong choice of $\theta$. The same situation arises when having
a strong prior belief over $\theta$, which is a common technique in Bayesian
modeling for incorporating knowledge about the domain. The knowledge is formalized
as probability distribution over the possible outcomes, as for instance
in our robotic experiments in Sec.~\ref{sec:robot-experiment}.
As detailed below, the attempt
to further minimize the entropy of $p(\theta|D,x,y)$ in such a situation
may lead to suboptimal choices of $x_{NE}$ that confirm the current
belief instead of challenging it.

This is obviously undesirable, instead we want a robust belief that cannot be
changed much by future observations. We
therefore want to induce the biggest change possible with every added
observation. In that way, we avoid local minima that occur if a belief is wrongly
biased. While minimizing the entropy would in this situation avoid observations
that change the belief, measuring the change of the belief regards an increase
of entropy as a desirable outcome.

\begin{figure}[t]
  \centering
  \includegraphics[width=\columnwidth]{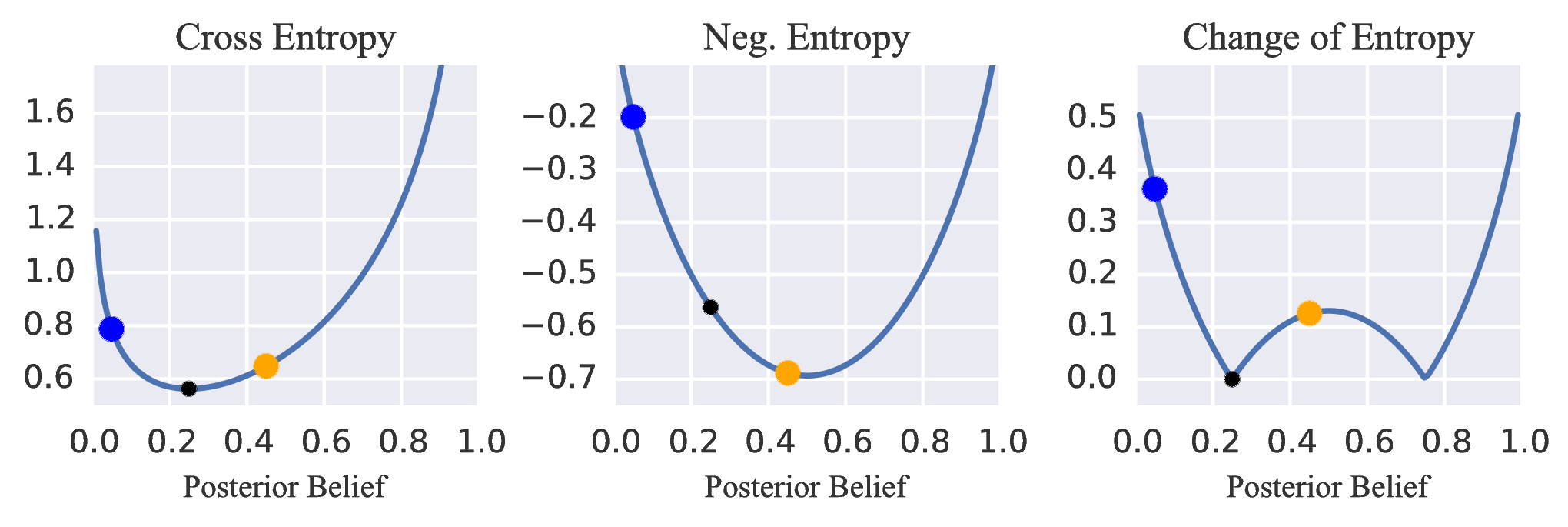}
  \caption{Characteristics of three different criteria for choosing samples:
    Cross entropy, neg.~entropy, and change of entropy. The belief is over a
    binary variable. The black dot indicates the prior belief of 0.25, the blue
    and yellow dot indicate the posterior after having seen two different
    observations. \textbf{Cross entropy} regards a change in any direction as an
    improvement. \textbf{Neg.~entropy}, in contrast, prefers changes that
    support the current belief over those that challenge it -- unless the
    posterior belief flips to having an even lower entropy than the
    prior. \textbf{Change of entropy} is similar to cross entropy in that for
    small changes it regards any direction as an improvement. For larger
    changes, however, is has a local optimum for a flat posterior of 0.5 and a
    local minimum for a flipped posterior with the same entropy as the prior.}
  \label{fig:d4huJ7eHbT}
\end{figure}

While a naive approach could be to maximize the expected change of the entropy
\begin{align}
  x^* &= \argmax_x \int_y p(y|x,D) \, \Big| H[p(\theta|y,x,D)] - H[p(\theta|D)]
  \Big|
\end{align}
this criterion has two undesirable pathologies (1)~it always has a local maximum
for a flat posterior belief with maximum entropy---unless the prior is already
flat---and (2)~changing a strong belief, say 0.25/0.75 for a binary hyper
parameter, to the equally strong but contradictory belief of 0.75/0.25 is one of
the global minima with zero change of entropy (see \figref{fig:d4huJ7eHbT}).

Another criterion that measures the change of the belief is the cross entropy
between the current and the expected belief.
This can be seen in \figref{fig:zrJICWGCGP}. Whereas the neg.~entropy is the
same for all prior beliefs, the cross entropy is high, when prior and posterior
belief disagree.
Intuitively neg.~entropy actually does not measure the \emph{change} of
distributions, but only the information of the posterior belief $p(\theta|D,x,y)$.

We therefore propose the \emph{\MaxCE{}} strategy which maximizes the expected
\emph{cross entropy} between the prior hypotheses belief $p(\theta|D)$ and
the posterior hypotheses belief $p(\theta|D,x,y)$.  This, again, can be
transformed to maximizing the KL-divergence, but now with switched
arguments (see again App.~\ref{sec:dkl-transformations} for details).

\begin{align}
x_{CE}
 &= \argmax_x \int_y p(y|x,D)~ H[p(\theta|D);p(\theta|D,x,y)] \label{eq:expected_cross_entropy}\\
 &= \argmax_x \int_y p(y|x,D)~ \kld{p(\theta|D)}{p(\theta|D,x,y)}, \label{eq:expected_inverse_KLD}
\end{align}
where $H[p(z),q(z)] = -\int_z p(z) \log q(z)$.

The KL-divergence
$\kld{p(\theta|D)}{p(\theta|D,x,y)}$ literally quantifies the additional
information captured in $p(\theta|D,x,y)$ relative to the previous
knowledge $p(\theta|D)$.

This does not require the entropy to decrease: the expected divergence $\kld{p(\theta|D)}{p(\theta|D,x,y)}$ 
can be high even if the expected entropy of the distribution $p(\theta|D,x,y)$ is higher than $H[p(\theta|D)]$---so
our criterion is not the same as minimizing expected model entropy. In
comparison to the KL-divergence formulation \eqref{eq:expected_KLD} of
expected entropy the two arguments are switched. The following
example and the later quantitative experiments will demonstrate the
effect of this difference.

\subsection{ An Example for Maximizing Cross Entropy Where Minimizing Entropy Gets Trapped}
\label{sec:min-entropy}

Bayesian experimental design suggests to minimize the expected entropy of the
model distribution \eqref{eq:exp_entropy}.  As we stated
earlier, this may lead to getting stuck in local optima for an iterative scenario.
We now explicitly show an example of such a situation. Assume a regression
scenario where two GP hypotheses should approximate a ground truth function.
Both GPs use a squared exponential kernel, but have a different length scale
hyper parameter. One of these GPs is the correct underlying model.

\begin{figure}[t]
  \centering
  \includegraphics[width=0.8\columnwidth]{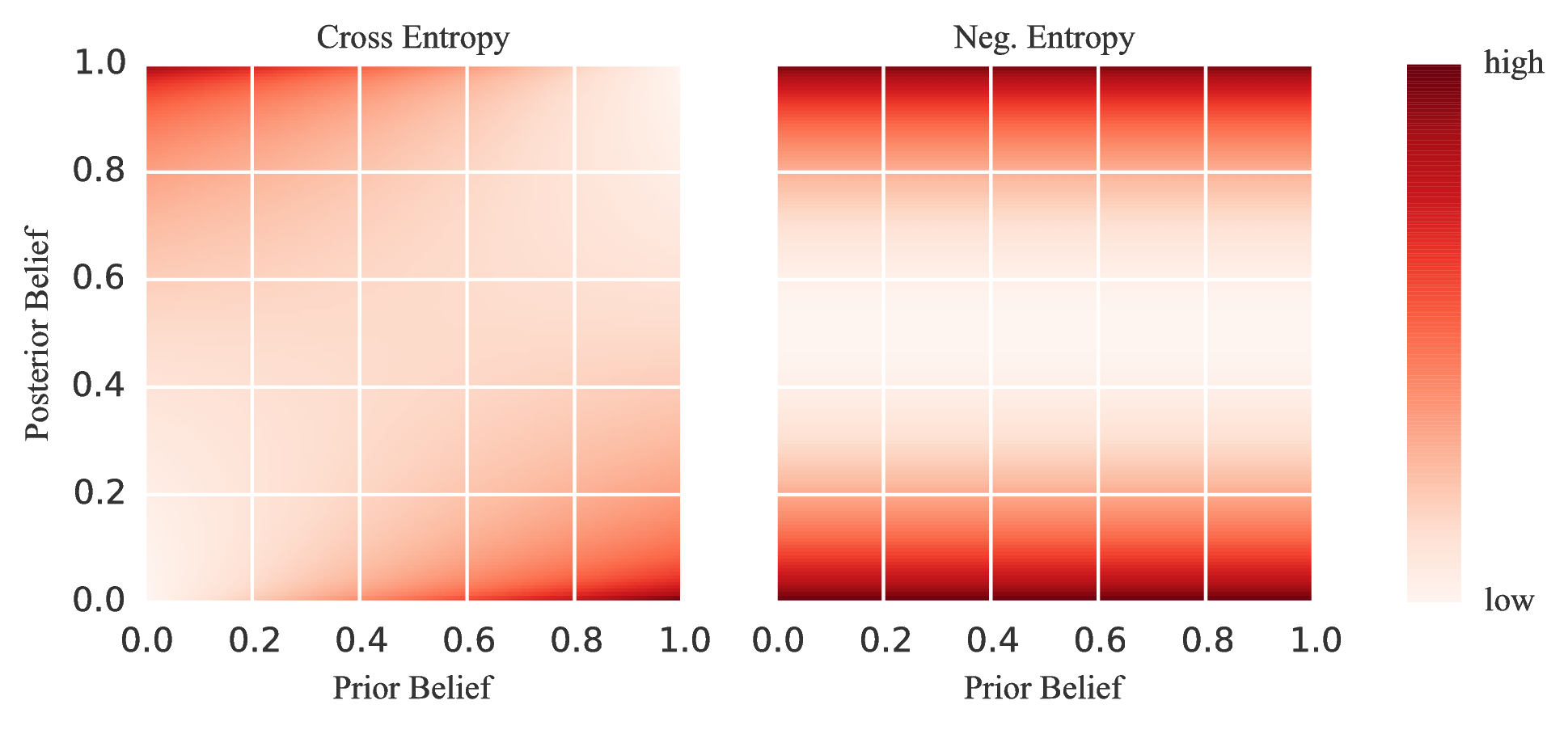}
  \caption{Cross entropy and neg.~entropy as a function of prior and posterior
    belief of two possible hypotheses. Values are normalized and the axes show
    the probability of one of the two hypotheses. Maximizing the cross entropy
    prefers a high entropy only if reached by a \emph{change} of the belief
    while maximizing the neg.~entropy ignores the prior belief.}
  \label{fig:zrJICWGCGP}
\end{figure}

\begin{figure}[th]
  \centering
  \includegraphics[width=.6\columnwidth]{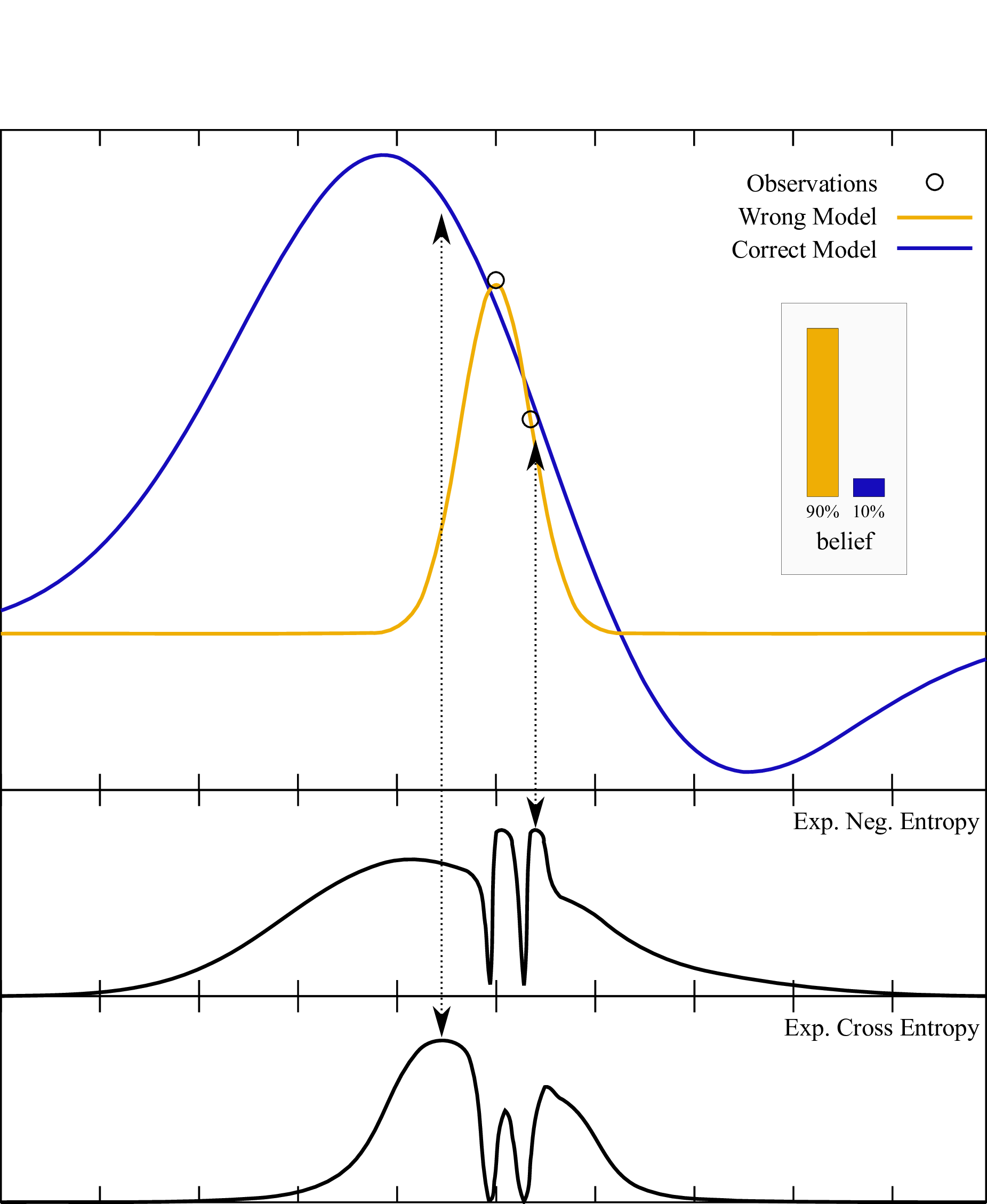}
  \caption{
  \label{fig:missleading}
  The top graph shows two competing hypotheses where one corresponds to the
  correct model (the Gaussian process the data are actually drawn from) and the
  other is wrong (a Gaussian process with a narrower kernel). For the two
  observations seen so far, the current belief is biased towards the wrong model
  because it is the more flexible one. The two curves below correspond the
  expected neg.~entropy \eqref{eq:exp_entropy} and the expected cross
  entropy \eqref{eq:expected_KLD} of the belief after performing a query at
  the corresponding location. The arrows indicate the query location following
  each of the two objectives.}
\end{figure}

Consider now a case where the first two observations by chance support the wrong
hypothesis. This may happen due to the fact that the ground truth function that
is actually sampled is itself only a random sample from the prior over all
functions described by the underlying GP. Furthermore observations may be noisy,
which may lead to a similar effect. Such a scenario---one that actually occurred
in our experiments---is shown in \figref{fig:missleading}. The probability for
the wrong model in this scenario is already around $90\%$. If we now compute the
expected neg.~entropy from \eqref{eq:exp_entropy} it has its maximum very close
to the samples we already got. This is due to the fact that samples possibly
supporting the other---the correct---model would temporarily decrease the
neg.~entropy. It would only increase again if the augmented posterior
actually flipped and the probability for the correct model got higher then
$90\%$.

The \MaxCE{} approach of maximizing cross entropy (see
\eqref{eq:expected_cross_entropy}) on the other hand favors
\emph{changes} of the hypotheses posterior in any direction, not only to
lower entropy, and therefore recovers much faster from the misleading
first samples. \figref{fig:missleading} shows both objectives for this
explicit example.

\subsection{The Conditional (Posterior) Hypotheses Entropy is not
  Submodular}
\label{sec:submodular}

While at the first glance this might contradict the finding that the entropy is
submodular \citep{fujishige1978polymatroidal} and optimizing submodular
functions can be done efficiently \citep{nemhauser1978analysis, iwata2001combinatorial},
we want to assure that it does not interfere with these facts. The
submodular entropy function is a set function on set of random variables, where
the entropy of the joint distribution of all variables in the set is computed.
Formally if $\Omega = \{V_1, \dots, V_n\}$ is a set of random variables, than for any $S
\subseteq \Omega$ the entropy of this subset $H(S)$ is submodular.
In contrast, we compute the entropy of the distribution of a fixed random variable,
conditioned on a set of random variables (see \eqref{eq:neg_entropy}).
As noted earlier this is the conditional entropy. In
App.~\ref{sec:conditional_entropy_is_not_submodular} we proof that this is not
submodular. The conditional entropy is, however, monotone. This means the
\emph{expectation} of the entropy decreases. For particular values of the
variables in $S$ it might increase.

\section{Comparison to Active Learning Strategies}
\subsection{Hypotheses Belief and Predictive Belief}
\label{sec:MSvsPS}

As opposed to existing active learning methods we define the objective
function directly in terms of the \emph{hypotheses belief} $P(\theta|D)$
instead of the \emph{predictive belief} $P(y|x,D)$.
We call $P(\theta|D,x,y)$ the \emph{posterior} hypotheses
belief after we have seen an additional data point $(x,y)$. Accordingly we call
$P(\theta|D)$ the \emph{prior} hypotheses belief, even though it is already
conditioned on observed data. It does, however, play the role of a Bayesian
prior in computing the posterior hypotheses belief. The
relation between the hypotheses posterior and predictive belief is
\begin{align}
  \begin{split}
    \underbrace{p(\theta|D,x,y)}_{\mathrm{hypotheses~belief}}
    = \frac{p(x,y,D|\theta) p(\theta)}{p(x,y,D)}
    \propto \underbrace{p(y|x,\theta,D)}_{\mathrm{predictive~belief}}~ p(D|\theta) ~.
  \end{split}
\end{align}

Minimizing the expected entropy on the
predictive belief is the direct
translation from Bayesian experimental design to the predictive belief:

\begin{align}
  x_{SI} = \argmax_{x} \int_y p(y|x, \theta, D) H[p(y|x, \theta, D)]
\end{align}

In the case of Gaussian distributions of $P(y|x, \theta, D)$ this is the same as
minimizing the expected variance of the predictive belief, since
$H[\mathcal{N}(\mu, \sigma^2)] = \frac{1}{2}\log(2\pi e \sigma^2)$, which is a
strictly monotonically increasing function on $\sigma^2$. Minimizing the
expected mean variance over the whole predictive space is introduced as active
learning criterion by \cite{cohn-ghahramani-jordan:96-jair}.

An even simpler but widely used technique is \emph{uncertainty sampling}
\citep{lewis-gale:94-sigir}. This techniques samples at points of high
uncertainty, measured in variance or entropy of the predictive belief. No
expectation is computed here. The assumption is that samples at regions with
high uncertainty will reduce the uncertainty most \citep{sebastiani1997bayesian}.

Normally, uncertainty sampling is used to train a single model, i.e. only one
hypothesis is assumend, while for
comparing it to our methods we have to consider a set of hypotheses. The most
natural way seems to handle the set of models as a mixture model and then
minimize the variance of this mixture model
\begin{align}
  E_\theta\left[P(y|x,\theta,D) - E_\theta\left[P(y|x,\theta,D)\right] \right]
  \label{eq:us_objective}
\end{align}
where $E_\theta[\cdot]$ is the expectation over all models.

A mix between both worlds is \emph{Query-by-Committee} (QBC)
\citep{seung-et-al:92-colt,mccallum-nigam:98-icml}. While aiming at
discriminating between different hypotheses, it uses the predictive belief for
measurements. It works  as follows: QBC handles a set of hypotheses, the
committee. When querying a new sample it chooses the sample with the largest
disagreement among the committee members. These samples are considered to be
most informative, since large parts of the committee are certainly wrong.

In a binary classification scenario disagreement is easy to determine. In
multi-class or regression scenarios it is harder to define. One approach, as
suggested by \cite{mccallum-nigam:98-icml}, is to use as measure of
disagreement the sum of KL-divergences from each committee
member's predictive belief to the mean of all committee member's
predictive beliefs.
\begin{align}
  \frac{1}{\|\theta\|} \sum_{\theta} D_{KL}\left(P(y|\theta,D,x) \, \bigg\| \, \frac{\sum_\theta P(y|\theta,D,x)}{\|\theta\|}\right)
\end{align}
While they assign a uniform prior over hypotheses in every step, we can also
compute the posterior hypotheses belief and use it is to weight the average:
\begin{align}
    \sum_\theta  P(\theta|D) \cdot
    D_{KL}\!\!\left(\!P(y|\theta,D,x) \bigg\| \sum_\theta P(\theta|D) P(y|\theta,D,x)\!\right)
  \label{eq:qbc_objective}
\end{align}

\subsection{Mixing Active Learning and Information Gathering}

While measuring the expected cross entropy is a good measure to find samples
holding information about latent model parameters of competing hypothesis it
might actually not query points that lead to minimal predictive error. For
example, regions that are important for prediction but do not discriminate
between hypothesis would not be sampled. Nevertheless, information about latent
model parameters may help to increase the predictive performance as well. For our
experiments we therefore additionally tested a linear combination of the \MaxCE{}
measure $f_{CE}$ from \eqref{eq:expected_cross_entropy} with the uncertainty
sampling measure $f_{US}$ from \eqref{eq:us_objective} in a combined objective
function $f_{mix}$
\begin{align}
  f_{mix} = \alpha \cdot f_{CE} + (1-\alpha) \cdot f_{US}~. \label{eq:mixture_objective}
\end{align}

\section{Experiments: Regression and Classification}
\label{sec:experiments}

Tasks that occur in real world scenarios can often be classified as either
regression (predicting a function value) or classification (predicting a class
label) tasks. We tested both task classes on
synthetic data. The regression scenario we also tested on a real world data set.
Typically one is interested in prediction performance. However, finding the
correct hypothesis might help for that task as well as generalizing to further
situations. We tested both in our experiments.

\subsection{Compared Methods}

We compared six different strategies: our \MaxCE{}, which maximizes the expected
cross entropy (see \eqref{eq:expected_cross_entropy}); classical Bayesian
experimental design, which minimizes the expected entropy (see
\eqref{eq:neg_entropy}); query-by-committee which optimizes Kullback-Leibler to
the mean (see \eqref{eq:qbc_objective}); uncertainty sampling (see
\eqref{eq:us_objective}), and random sampling, which randomly choses the next
sample point. Additionally we tested a mixture of \MaxCE{} and uncertainty
sampling (see \eqref{eq:mixture_objective}). The mixing coefficient, which was
found by a series of trial runs, was $\alpha=0.5$ for both synthetic data sets
and $\alpha=0.3$ for the CT slices data set.

\newcommand{\figureScaling}{0.45}

\begin{figure*}\center
  \subfigure[Regression Model Entropy] {
  \includegraphics[width=\figureScaling\columnwidth]{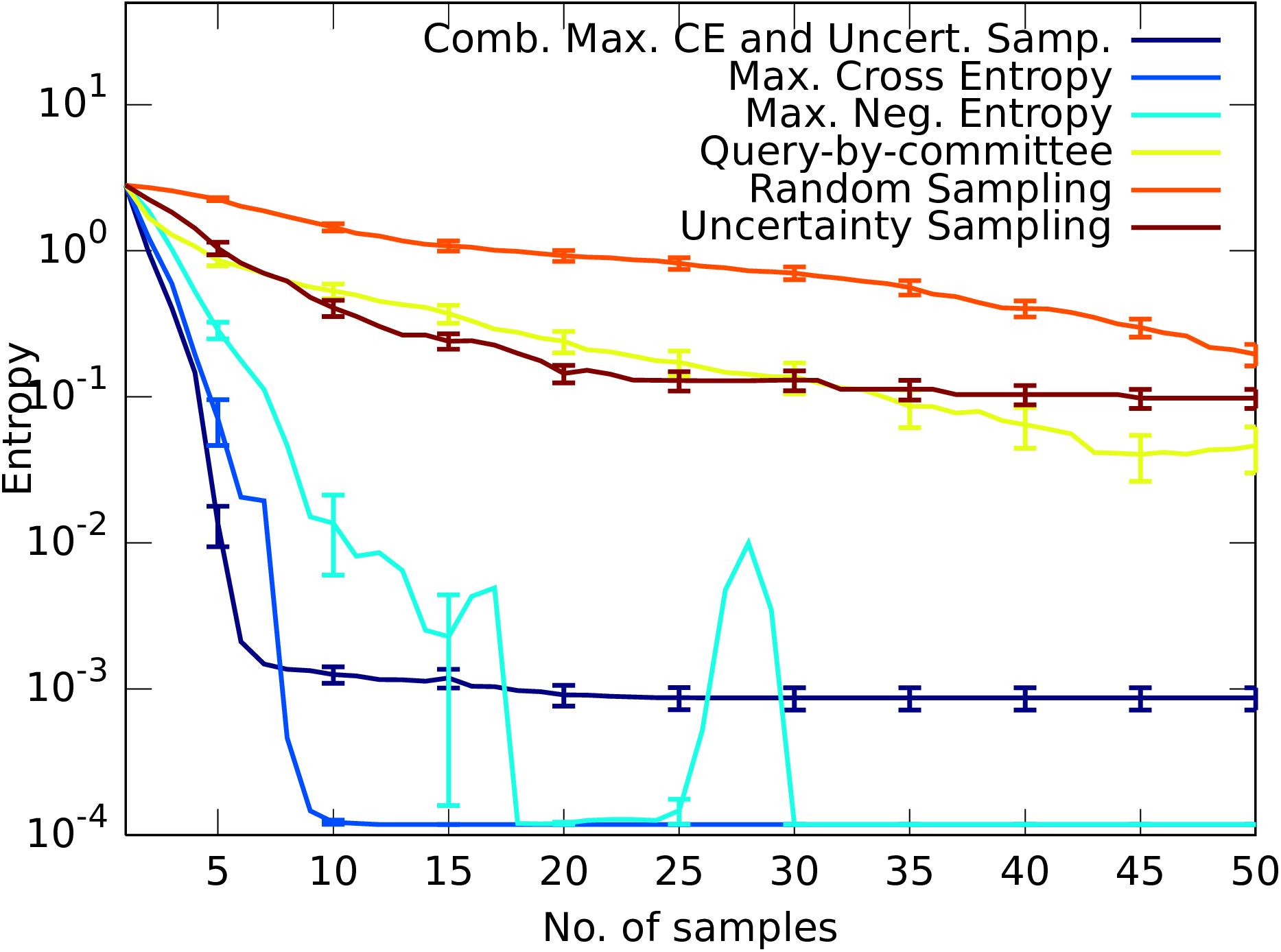}
  \label{fig:mean_performance_ent2}
  }
  \subfigure[Regression MSE] {
  \includegraphics[width=\figureScaling\columnwidth]{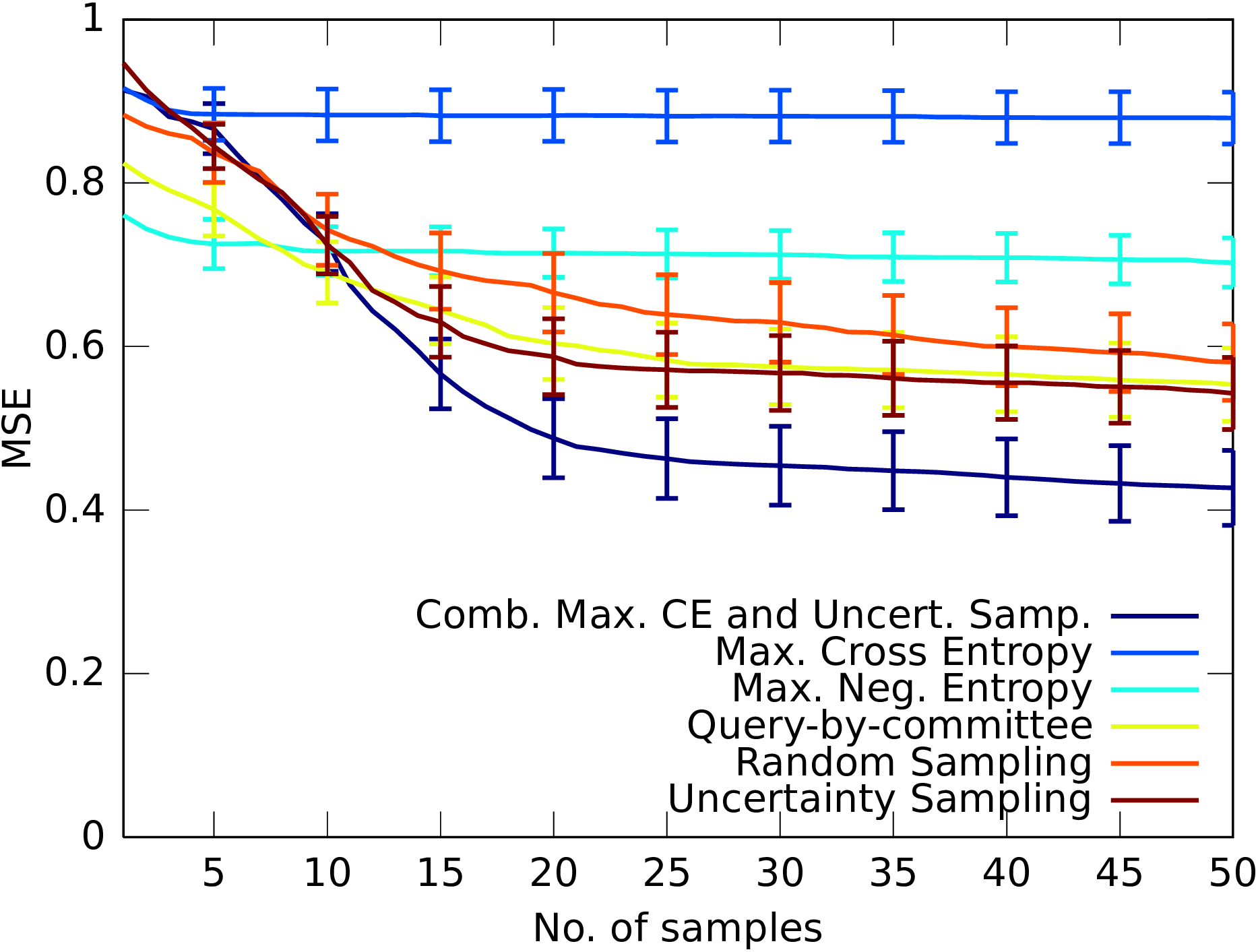}
  \label{fig:mean_performance_acc2}
  }
  \\

  \subfigure[Classification Model Entropy] {
  \includegraphics[width=\figureScaling\columnwidth]{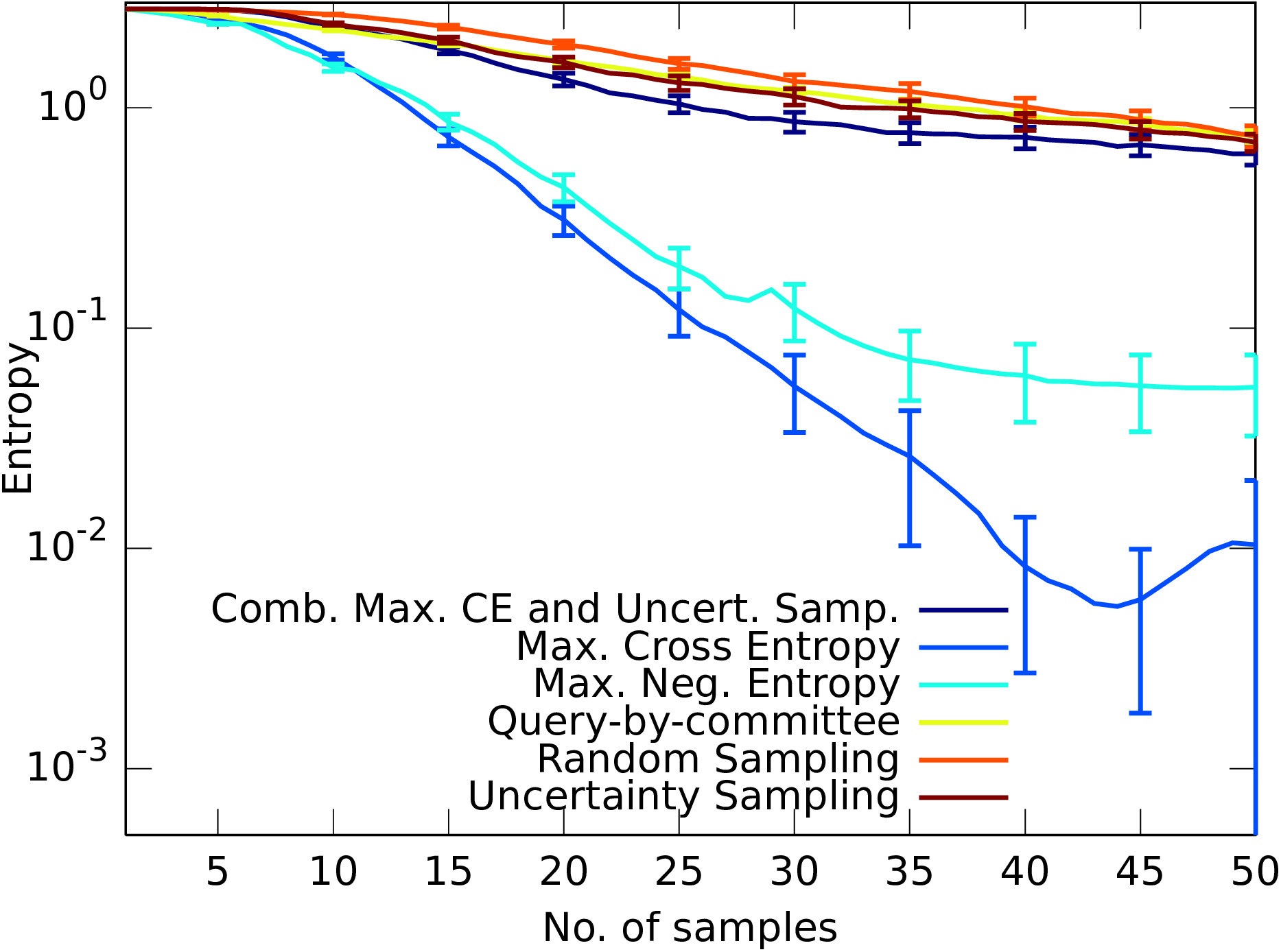}
  \label{fig:mean_performance_ent1}
  }
  \subfigure[Classification Accuracy] {
  \includegraphics[width=\figureScaling\columnwidth]{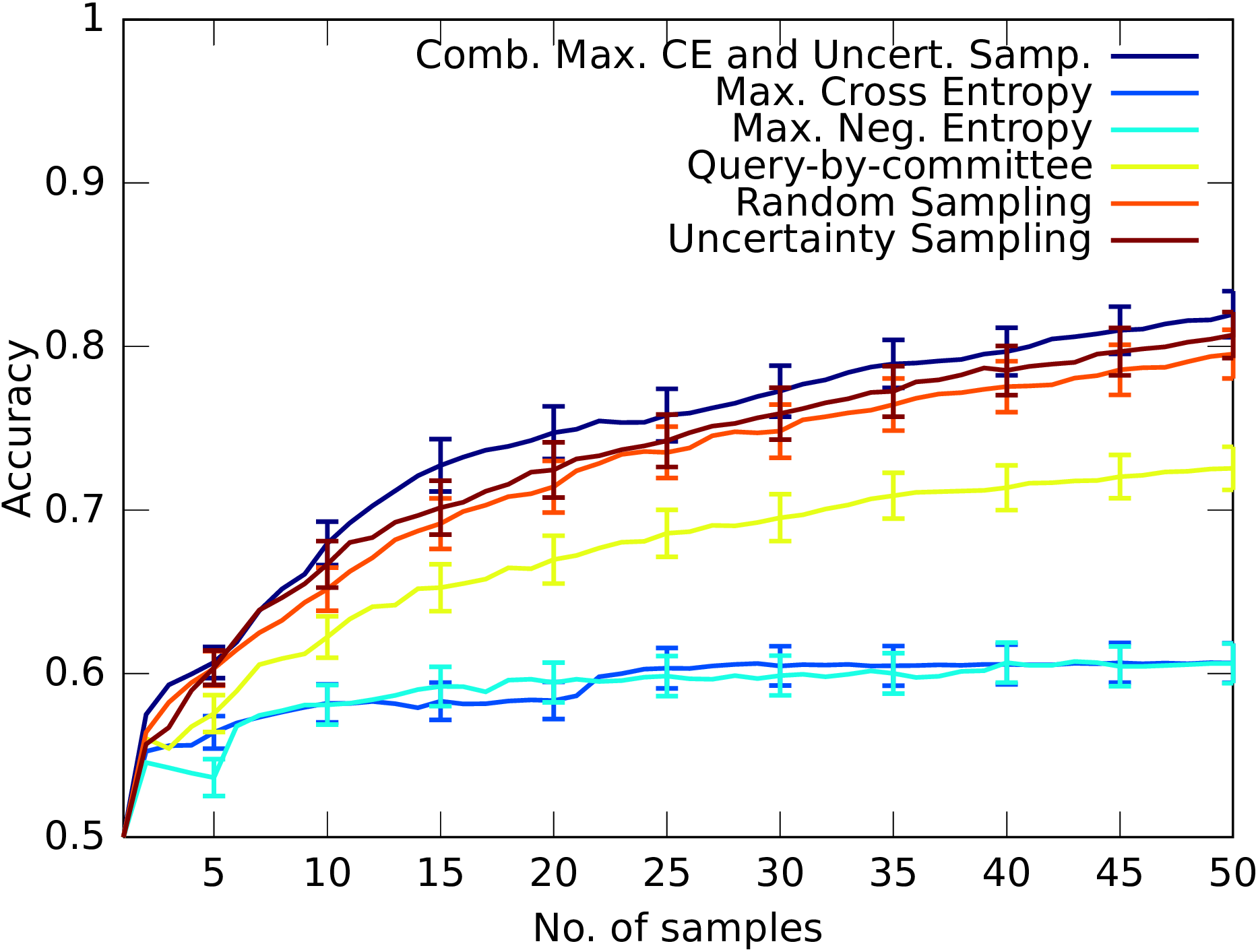}
  \label{fig:mean_performance_acc1}
  }
  \\
  \subfigure[CT Slices Model Entropy] {
  \includegraphics[width=\figureScaling\columnwidth]{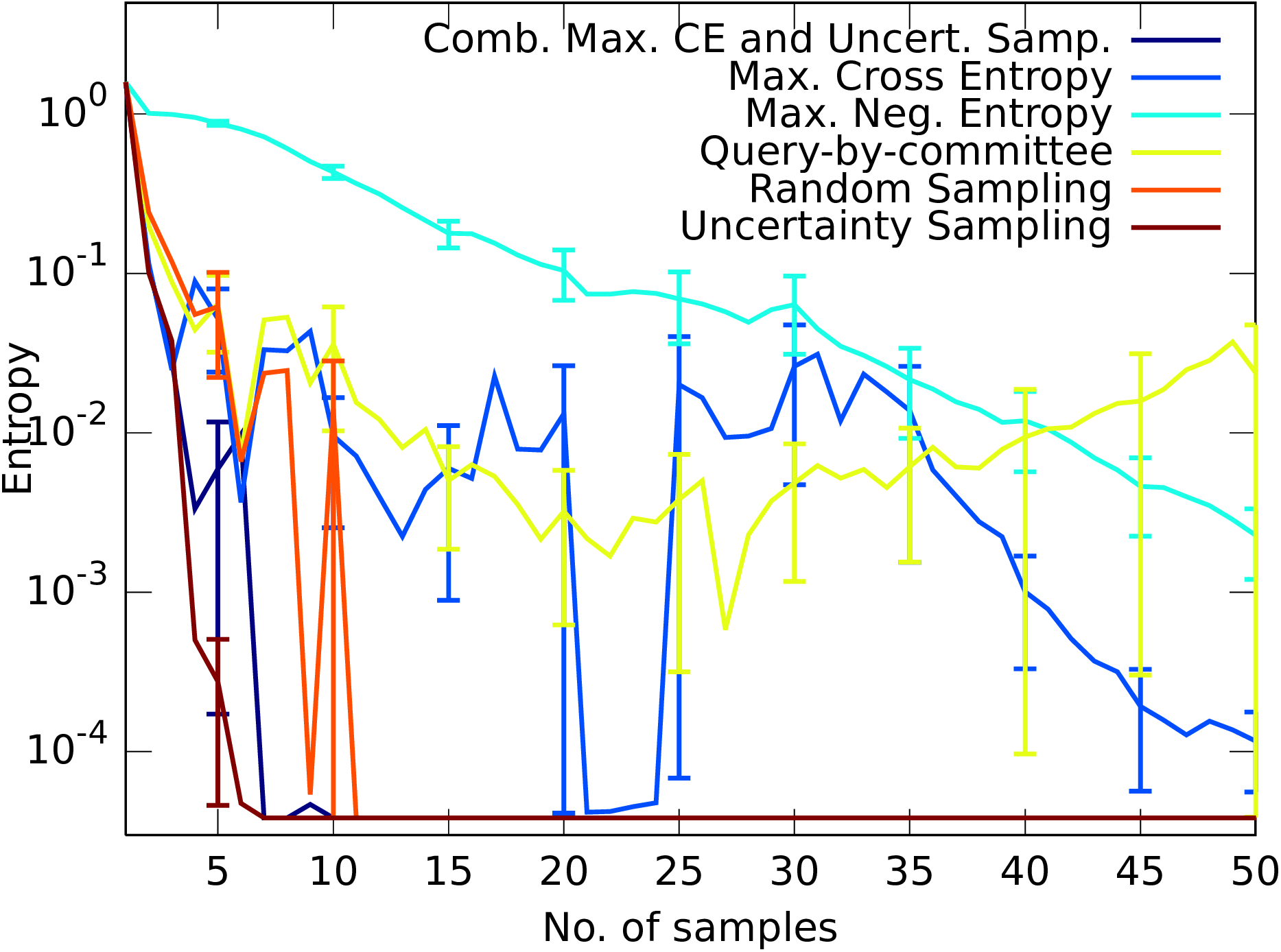}
  \label{fig:mean_performance_ent3}
  }
  \subfigure[CT Slices Error] {
  \includegraphics[width=\figureScaling\columnwidth]{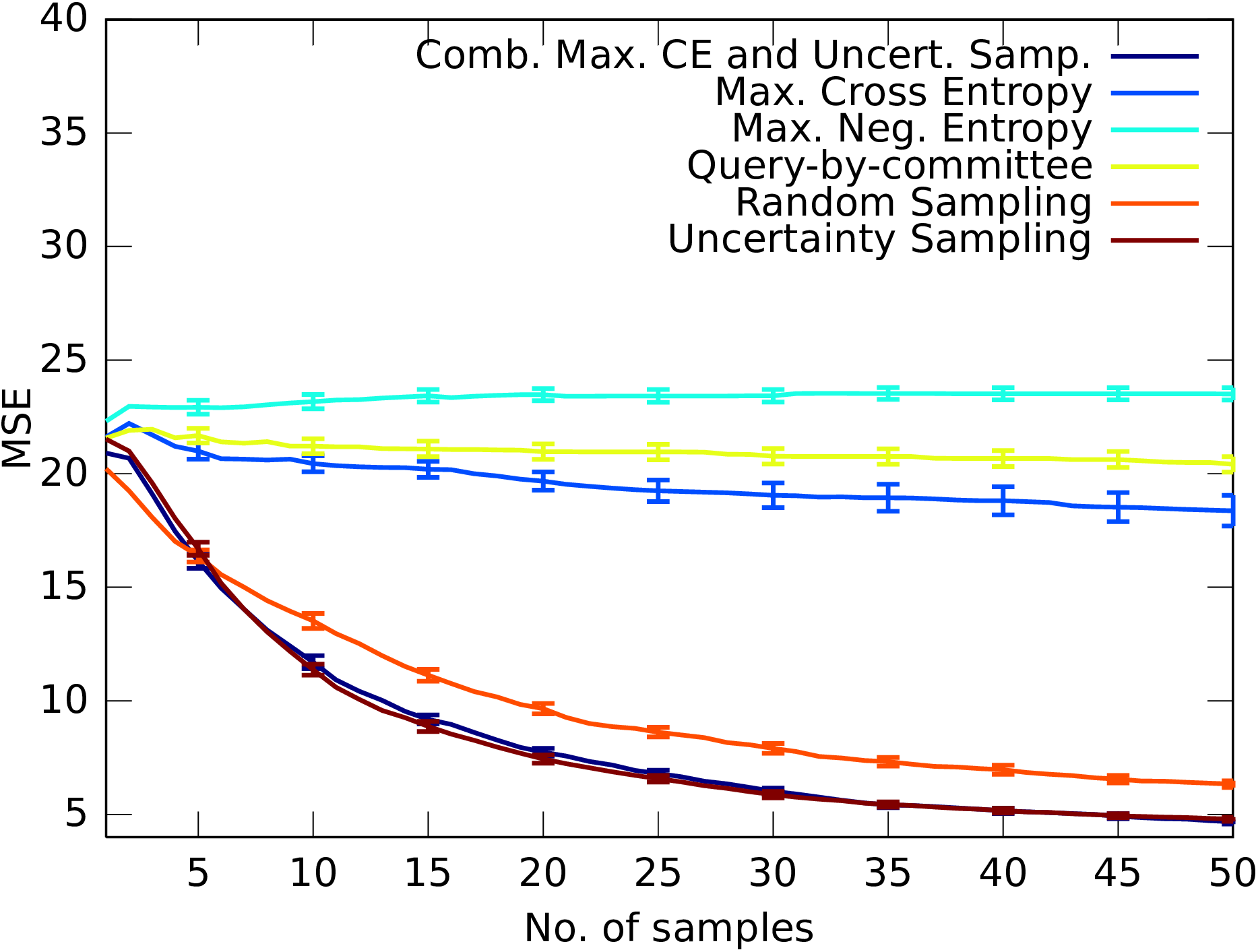}
  \label{fig:mean_performance_acc3}
  }
  \caption{ The mean performance of the different methods for the classification
  tand regression tasks.}
  \label{fig:mean_performance}
\end{figure*}

\subsection{Measures}

To measure progress in discriminating between hypotheses we computed the entropy
of the posterior hypotheses
belief for each method. To measure progress in the predictive performance we
plot the classification accuracy and the mean squared error for classification
and regression, respectively. To compute an overall predictive performance for a
method we took the weighted average over the different models, with the
posterior probabilities as weights. This corresponds to the maximum a posteriori
estimate of the marginal prediction
\begin{align}
  p(y|D,X) = \sum_\theta p(\theta|D)~ p(y|\theta,D,X) ~.
\end{align}
Fig.~\ref{fig:mean_performance} show these measures for all our experiments.

\subsection{Synthetic Data}

We tested our method in both a 3D-regression and a 3D-classification
task. The setup for both experiments was essentially the same: A
ground truth Gaussian Process (GP) was used to generate the data. The
kernel of the ground truth GP was randomly chosen to depend either on
all three dimensions $(x, y, z)$, only a subset of two dimensions $(x,
y)$, $(y, z)$ or $(x, z)$, or on only one dimension $(x)$, $(y)$ or
$(z)$. Finding the correct hypothesis in this case corresponds to a feature selection
problem: uncovering on which features the unknown true GP depends
on. The latent variable $\theta$, to be uncovered by the active learning
strategies, enumerates exactly those seven possibilities. One run
consisted of each method independently choosing fifty queries
one-by-one from the same ground truth model. After each query the
corresponding candidate GP was updated and the hypotheses posterior was
computed.

Fig.~\ref{fig:mean_performance_ent2}, \ref{fig:mean_performance_acc2},
\ref{fig:mean_performance_ent1} and \ref{fig:mean_performance_acc1} show the mean
performance over 100 runs of the synthetic classification and
regression tasks, respectively. Since we average over 100 runs, the
error bars of the mean estimators are very small. Both hypotheses belief entropy
and accuracy/mean squared error are shown.

On this synthetic data \MaxCE{} significantly outperforms all other
tested methods in terms of entropy, followed by Bayesian
experimental design, and the mixture of \MaxCE{} and uncertainty
sampling (\figref{fig:mean_performance_ent2}
and~\ref{fig:mean_performance_ent1}). As expected, In terms of
classification accuracy and predictive error both \MaxCE{} and
Bayesian experimental design perform poorly. This is because their
objectives are not designed for prediction but for hypothesis discrimination.
However, the mixture of \MaxCE{} and uncertainty sampling,
performs best (\figref{fig:mean_performance_acc2}
and~\ref{fig:mean_performance_acc1}), which is presumably due to its
capability to uncover the correct hypothesis quickly.

\subsection{CT-Slice Data}

We also tested our methods on a high dimensional (384 dimensions) real
world data set from the machine learning repository of the University
of California, Irvine \citep{uci-ml-repo}. The task on this set is to
find the relative position of a computer tomography (CT) slice in the
human body based on two histograms measuring the position of bone (240
dimensions) and gas (144 dimensions). We used three GPs with three
different kernels: a $\gamma$-exponential kernel with $\gamma = 0.4$,
an exponential kernel, and a squared exponential kernel.  Although
obviously none of these processes generated the data we try to find
the best matching process alongside with a good regression result.
Fig.~\ref{fig:mean_performance_ent3} and \ref{fig:mean_performance_acc3}
show the mean performance over 40 runs on the CT slice data set.

In the CT slice data set neither \MaxCE{} nor Bayesian experimental design minimize the entropy quickly
(\figref{fig:mean_performance_ent3}). This may be a consequence of the
true model \emph{not} being among the available alternatives. As a
consequence both methods continuously challenge the belief thereby
preventing it from converging. QBC may be subject to the same
struggle, here even resulting in an increase of entropy after the
first 25 samples. In contrast for uncertainty sampling, our mixture
method, and random sampling the entropy converges reliably. Concerning
the predictive performance \MaxCE{}, Bayesian experimental design, and
QBC do not improve noticeably over time (cf.~explanation above). Again
uncertainty sampling and our mixture method perform much better, while
here the difference between them is not significant.

\section{Robot Experiment: Joint Dependency Structure Learning}
\label{sec:robot-experiment}

In another experiment we used the \MaxCE{} method to uncover the structure of
dependencies between different joints in the environment of a robot.
Consider a robot entering an unknown room. To solve tasks successfully it is
necessary to explore the environment for joints that are controllable by the
robot, such that it is able to e.g. open drawers, push buttons or unlock a door.
In earlier work we have shown how such exploration can be driven by information
theoretic measures \citep{otte-et-al:14-iros}.
Many joints are however dependent
on each other, such as keys can lock cupboards or handles need to be pressed
before a door can be opened.
We modeled these dependencies with a
probabilistic model that captures the insight that many real world mechanisms
are equipped with some sort of feedback, as for instance a force raster or
click-sounds that support the use of the mechanisms to find dependencies more
quickly \citep{kulick-et-al:15-icra}. For the details on the model of feedback
we refer to that publication. Here we show a simplified model, necessary to
follow the introduction of \MaxCE{} in this context.
Fig.~\ref{fig:jds-model} shows the simplified graphical model.

\begin{figure}[th]
  \centering
  \includegraphics[width=.25\textwidth]{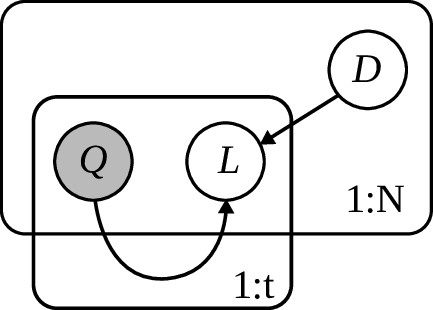}
  \caption{A simplified version of the graphical model from
  \cite{kulick-et-al:15-icra}, omitting the feedback of the explored object. The
  latent distribution to be learned is $P(D^j)$, which captures the dependency
  structure as discrete distribution.}
  \label{fig:jds-model}
\end{figure}

Consider an environment with $N$ joints, where each joint might be locked or
unlocked over time. The locking state might be dependent on the position of the
other joints. Let $Q^j_t$, $L^j_t$ and $D^j$ be random variables. $Q^j_t$ is the
joint state of the $j$-th joint in the environment at time $t$. $L^j_t$ is the
locking state of the $j$-th joint at time $t$ and $D^j$ is the dependency
structure of the $j$-th joint. $D^j$ is a discrete variable with domain $\{1,
\dots, N+1\}$. The $i$-th outcome indicates that joint $j$ is dependent on $i$,
whereas the last outcome indicates that the joint is independent from other
joints.

We now want to uncover the dependency structure of all joints. Thus we want to
know the distribution of all $D^j$. $D^j$ here is the latent variable we want to
gather information about (called
$\theta$ throughout the former parts of the paper). $Q^j_t$ and $L^j_t$ are the data
observed so far (i.e. $x$ and $y$ respectively). If we want to use \MaxCE{}
to learn about the dependency structure we need to compute the expected one-step
cross entropy between the current joint dependency structure $P_{D^j_t}$ and the
expected joint dependency structure distribution one time step ahead
$P_{D^j_{t+1}}$ and maximize this expectation to get the optimal next sample
position, corresponding directly to Eq.~(\ref{eq:expected_cross_entropy}):

\begin{align}
\label{eq:next_action}
  ({Q^{1:N}_{t+1}}^*, j) =
  \argmax_{(Q^{1:N}_{t+1}, j)} \sum_{L^{j}_{t+1}}
  P\left(L^{j}_{t+1}|Q^{1:N}_{t+1}\right)~\cdot H\left[ P_{D^j_{t}};P_{D^j_{t+1}} \right]
\end{align}
with
\begin{align}
  P_{D^j_{t}}   &= P\left(D^{j}|L^{j}_{1:t}, Q^{1:N}_{1:t}\right)\\
  P_{D^j_{t+1}} &= P\left(D^{j}|L^{j}_{1:t+1}, Q^{1:N}_{1:t+1}\right).
\end{align}

We conducted two versions of this experiment. A quantitative, but simulated
version of the experiment and second a qualitative real-world experiment on a
PR2 robot (see Fig.~\ref{fig:real-world}). In the simulated version the agent is
presented an environment with three randomly instantiated furnitures as
described in Table~\ref{tab:furniture}. In the real world experiment the PR2
robot has to uncover that a key is locking the drawer of an office cabinet.

\begin{figure}[th]
  \centering
  \includegraphics[width=.5\textwidth]{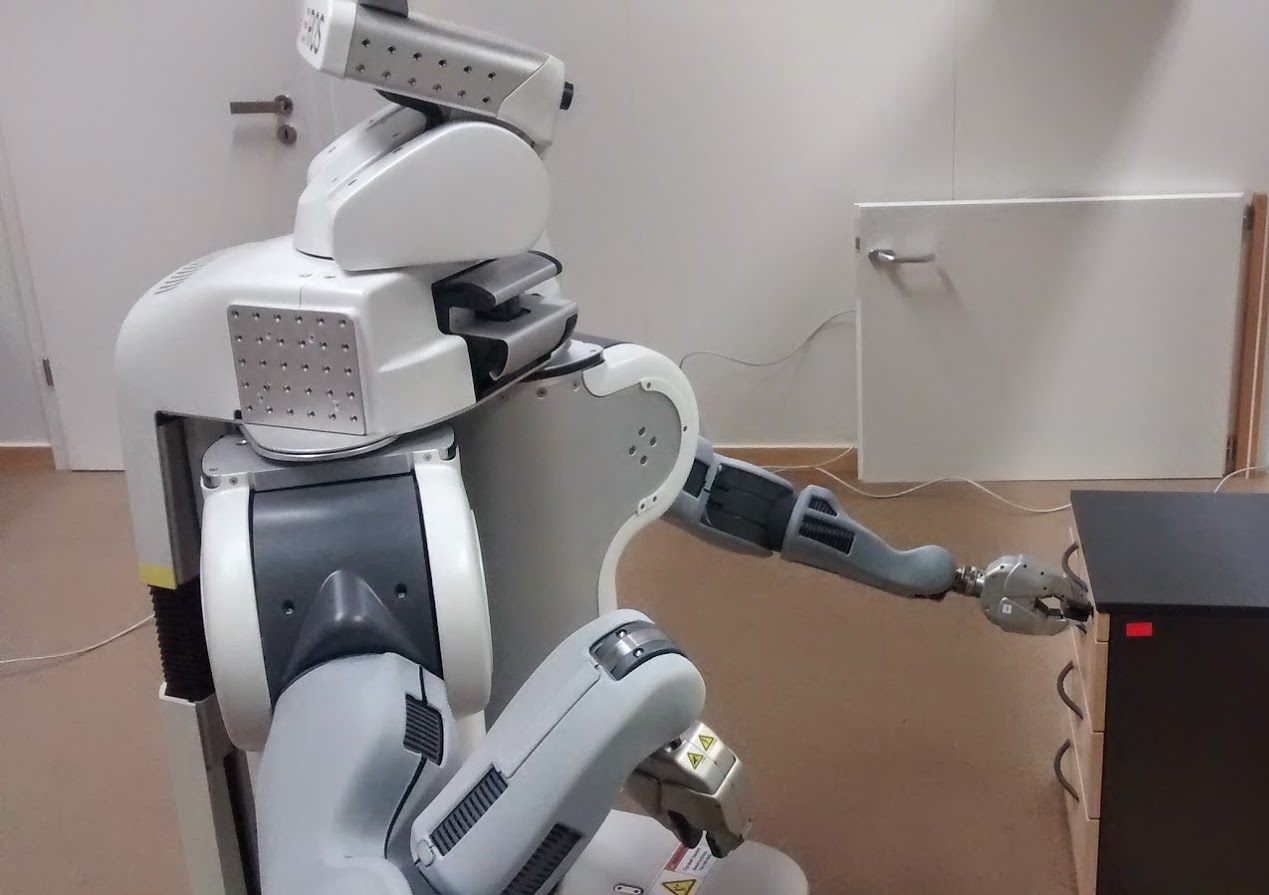}
  \caption{A PR2 robot tries to uncover the dependency structure of a typical
  office cabinet by exploring the joint space of the key and the drawer.}
  \label{fig:real-world}
\end{figure}

\begin{table}[ht]
  \centering
  \begin{tabularx}{\textwidth}{@{}lXX@{}}
    \hline
    \textbf{Name}        & \textbf{Description}                                &
    \textbf{Locking mechanism}\\ \hline
    Cupboard with handle & A cupboard with a door and a handle attached to it. & The handle must be at upper or lower limit to unlock the door \\
    Cupboard with lock   & A cupboard with a door and a key in a lock          & The key must be in a particular position to unlock the door \\
    Drawer with handle   & A drawer with a movable handle                      & The handle must be at upper or lower limit to unlock the drawer \\
    Drawer with lock     & A drawer with a key in a lock                       & The key must be in a particular position to unlock the drawer \\
    \hline
  \end{tabularx}
  \caption{Furniture used in the simulation.}
  \label{tab:furniture}
\end{table}

\subsection{Actions and Observations}

The robot can directly observe the joint state of all joint over time and can
move the joints to a desired position. At a given position the robot can ask an
oracle about the locking state of one joint.

\subsection{Prior}

For the dependency distribution $P(D^j)$ we choose the following prior:
\begin{align}
  P(D^j = i) =
  \begin{cases}
    0                    & \mathrm{~if~} i=j \quad \textrm{(self-dependent)}\\
    .7                   & \mathrm{~if~} i=N+1 \quad \textrm{(independent)}\\
    \frac{1}{d(i,j) c_N} & \mathrm{~else~} \quad \textrm{($j$ depends on $i$)}
  \end{cases}
\end{align}
with $d(i, j)$ being the (euclidean) distance between joint $i$ and $j$ and
$c_N$ being a normalization constant. This captures our intuition that most
joints are movable independently from the state of other joints, e.g. most
joints are not lockable etc. Additionally it models our knowledge that joints
that can lock each other are often close to each other. The hard zero prior for
self-dependence rules out the possibility of a joint locking itself.

\subsection{Results and Discussion of the Simulated Experiment}

We tested the \MaxCE{} method, expected neg.~entropy and a random strategy, each
50 times. As results we show in Fig.~\ref{fig:sim_results_entropy} two things.
First the sum of the entropies of all $P(D^j)$. Here one can see that only
\MaxCE{} is able to decrease the entropy significantly. As expected random
apparently performs worse than neg.~expected entropy. But this is a wrong
conclusion: In the second plot we show how many dependencies are classified
correctly, if we apply an (arbitrary) decision boundary at $0.5$. Neg.~expected
entropy is not able to classify anything correct, but the three independent
joints, which are already covered by the prior $P(D^j)$ (whereas random is able
to slowly uncover other joint dependencies). The strong prior---which arguably
is a reasonable one---let the classical Bayesian experimental design strategy
pick queries that do not uncover the true distribution but stays at the local
minimum.
The entropy increases in the random strategy, since the prior is already a
strong belief. So during the exploration of the joints the entropy first
increases and only later decreases again. The neg.~entropy strategy on the other
hand does not change the belief and thus keeps a lower entropy. Note that after
the first step also the cross entropy criterion has slightly increased the
entropy and only after three observations drop below the neg.~entropy strategy.

\begin{figure}[t]
  \centering
  \includegraphics[width=\textwidth]{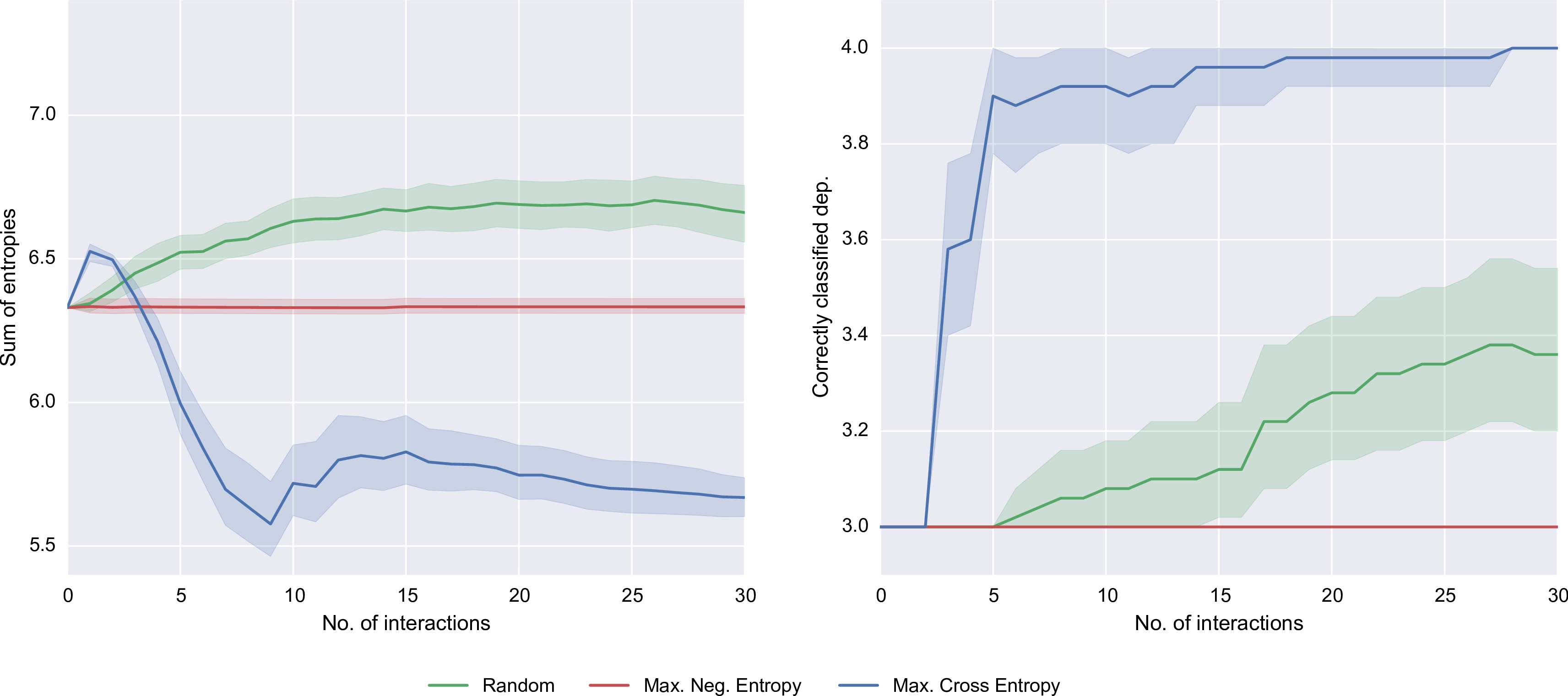}
  \caption{Results of simulation experiments. Left we show the sum of entropies
  over all dependency beliefs. Right we show the mean correctly classified
  joints with an arbitrary decision boundary at $0.5$. (Similar figure as in
  \citep{kulick-et-al:15-icra}.)}
  \label{fig:sim_results_entropy}
\end{figure}

\subsection{Results and Discussion of the Real World Experiment}

In the real world experiment we let the PR2 robot explore the office cabinet
with the \MaxCE{} strategy. It could identify the correct dependency structure
after a few interactions. We show the two $P(D^j)$ distributions in
Fig.~\ref{fig:robot_results}. Notably the distribution of the independent joint
doesn't change. This comes from the fact that the robot can find no strong evidence of
independence, as long as it does not have covered the whole joint space of the
other joint. To understand this note that the locking state from the key never
changes, i.e. it is always movable. So there is no evidence against the
possibility of a dependency from the drawer to the key, since there might be a
position of the drawer which locks the key. Only if the agent has seen every
possible state of the drawer it can be sure that the key is independent.
Since only a handful of drawer states are observed, the prior distribution
almost preserves during the whole experiment.

\begin{figure}[t]
  \centering
  \includegraphics[width=\textwidth]{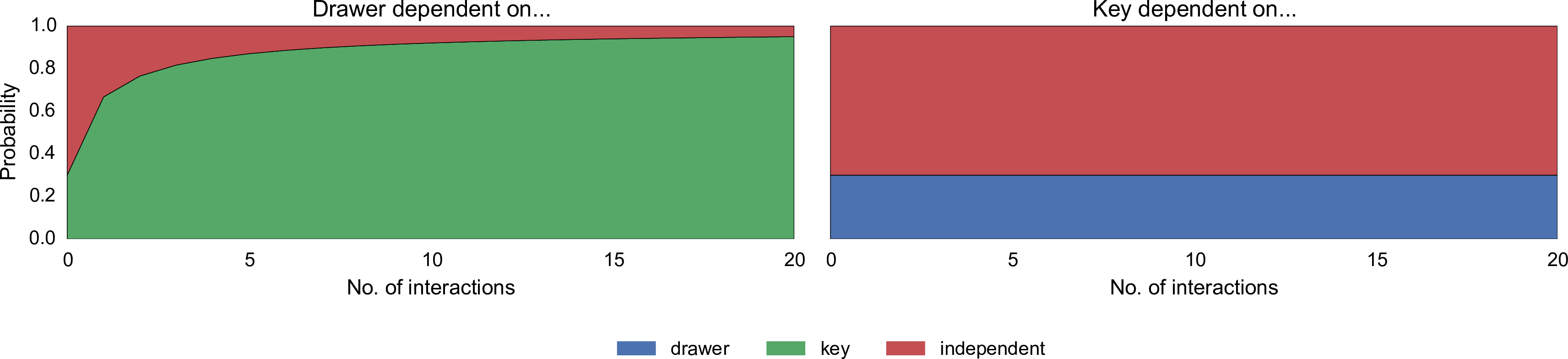}
  \caption{Results from the real world experiment. We show the belief over the
  dependency structure of both joints of the drawer. (Figure as in
  \citep{kulick-et-al:15-icra}.)}
  \label{fig:robot_results}
\end{figure}
\section{Conclusion and Outlook}
\label{sec:conclusion_and_outlook}

The presented results strongly suggest that our newly developed strategy of
maximizing the expected cross entropy is superior to classical Bayesian
experimental design for uncovering latent parameters in an iterative setting.

The results on predictive performance additionally demonstrate a successful
application of \MaxCE{} for prediction by mixing it with an uncertainty sampling
objective. The resulting objective at the same time actively learns the latent
parameters and accurate predictions. This initially goes at the expense of
accurate predictions but at some point more than compensates this
fall-back. This might be the case, because this way areas which are important
for false model hypothesis can be ignored and thus the right model is trained
better.

So far our mixing strategy is rather simple. But the results suggest that
the mixing helps. Investigating better mixing strategies might lead to more
improvements.

So far we only investigated the discrete case of $k$ distinct models. The same
techniques described in this paper may be useful to find samples to optimize
continuous hyper parameter. In this case the sum over models will become an
integral and efficient integration techniques need to be applied to the method
to keep it computationally tractable. It also might be applicable to leverage
the insight of \cite{ko-et-al:95-ops} that the entropy is submodular to
implement efficient approximations of the optimization.

Another direction of research would involve finding better optimization
techniques to find the actual maxima to up the process. When using GPs all
involved distributions are Gaussian (or approximated by Gaussians for the
classification case). As such they are infinitely differentiable, so higher
order methods might prove useful.

\section*{Acknowledgments}

The CT slices database was kindly provided by the UCI machine learning
repository \citep{uci-ml-repo}.  We thank Stefan Otte for help with the robot
experiments. Johannes Kulick was funded by the German Research Foundation (DFG,
grant TO409/9-1) within the priority programm ``Autonomous learning'' (SPP1597).
Robert Lieck was funded by the German National Academic Foundation.

\bibliography{references}
\bibliographystyle{plainnat}

\appendix
\section{Expected Kullback-Leibler divergence transformations}
\label{sec:dkl-transformations}

The KL-divergence, entropy, and cross-entropy of two distributions
$p$ and $q$ are closely related (rows in \eqref{eq:lJ4r81Nsga}) and can be
rewritten as expectation values (columns in \eqref{eq:lJ4r81Nsga})
\begin{align}
  \begin{array}{rl@{}ll@{~}l}
    \Dklbig{p(x)}{q(x)}
    &=&\HHbig{p(x),q(x)}
    &-&\HHbig{p(x)} \\[2mm]
    \makebox[30mm]{\rotatebox{90}{=}}
    && \makebox[10mm]{\rotatebox{90}{=}}
        && \makebox[10mm]{\rotatebox{90}{=}} \\[2mm]
    \displaystyle{\int p(x) \log \frac{p(x)}{q(x)} dx}
    &=-&\displaystyle{\int p(x) \log q(x)}
    &+&\displaystyle{\int p(x) \log p(x)} \\[2mm]
    \makebox[30mm]{\rotatebox{90}{=}}
    && \makebox[10mm]{\rotatebox{90}{=}}
        && \makebox[10mm]{\rotatebox{90}{=}} \\[2mm]
    \displaystyle{\eexpectBig{\log \frac{p(x)}{q(x)}}_{p(x)}}
    &=-&\displaystyle{\eexpectBig{\log q(x)}_{p(x)}}
    &+&\displaystyle{\eexpectBig{\log p(x)}_{p(x)}}~.
  \end{array}
        \label{eq:lJ4r81Nsga}
\end{align}
When taking the expectation of the KL-divergence over $p(y|x,D)$, depending on
the direction of the KL-divergence, either the entropy or the cross-entropy term
is constant with respect to $x$ (and therefore drops out when taking the
$\argmax_{x}$)
\begin{align}
    \eexpectBig{\Dklbig{&p(\theta|y,x,D)}{p(\theta|D)}}_{p(y|x,D)} = \nonumber \\
    &=
    - \eexpectBig{\HHbig{p(\theta|y,x,D)}}_{p(y|x,D)}
    + \eexpectBig{\HHbig{p(\theta|y,x,D),p(\theta|D)}}_{p(y|x,D)} \\
    &=
    - \eexpectBig{\HHbig{p(\theta|y,x,D)}}_{p(y|x,D)}
    - \eexpectBig{\log p(\theta|D)}_{p(\theta|y,x,D),p(y|x,D)} \label{eq:cNzwSeyFpR} \\
    &=
    - \eexpectBig{\HHbig{p(\theta|y,x,D)}}_{p(y|x,D)}
    - \eexpectBig{\log p(\theta|D)}_{p(\theta|D)} \label{eq:hNxZxFjo94} \\
    &=
    - \eexpectBig{\HHbig{p(\theta|y,x,D)}}_{p(y|x,D)}
    - \underbrace{\HHbig{p(\theta|D)}}_{const.} \\
  \eexpectBig{\Dklbig{&p(\theta|D)}{p(\theta|y,x,D)}}_{p(y|x,D)} = \nonumber \\
    &=
      \eexpectBig{\HHbig{p(\theta|D),p(\theta|y,x,D)}}_{p(y|x,D)}
      - \eexpectBig{\HHbig{p(\theta|D)}}_{p(y|x,D)} \\
    &=
      \eexpectBig{\HHbig{p(\theta|D),p(\theta|y,x,D)}}_{p(y|x,D)}
      - \underbrace{\HHbig{p(\theta|D)}}_{const.}~.
\end{align}
For the step from \eqref{eq:cNzwSeyFpR} to \eqref{eq:hNxZxFjo94}, note that
$p(\theta|x,D)=p(\theta|D)$ since $\theta$ is independent of $x$ so that for any function
$f(\theta,D)$ that depends only on $\theta$ and $D$, such as $\log p(\theta|D)$ above, an
expectation over $p(\theta|y,x,D)$ and $p(y|x,D)$ is equal to an expectation over just
$p(\theta|D)$
\begin{align}
  \eexpectBig{f(\theta,D)}_{p(\theta|y,x,D),p(y|x,D)}
  &= \iint_{y,\theta} f(\theta,D) \, p(\theta|y,x,D) \, p(y|x,D) \, dy \, d\theta \\
  &= \int_{\theta} f(\theta,D) \bigg[ \int_{y} \, p(\theta,y|x,D) \, dy \bigg] d\theta \\
  &= \int_{\theta} f(\theta,D) \, p(\theta|x,D) \, d\theta \\
  &= \int_{\theta} f(\theta,D) \, p(\theta|D) \, d\theta \\
  &= \eexpectBig{f(\theta,D)}_{p(\theta|D)}~.
\end{align}

\section{Conditional Entropy Is Not Submodular}
\label{sec:conditional_entropy_is_not_submodular}

\begin{definition}
For a set $\Omega$, the set function $f : 2^\Omega \rightarrow \mathbb{R}$ is submodular if and only if
\begin{align}
  f\big(D \cup \{y_1\}\big) + f\big(D \cup \{y_2\}\big) &\geq f\big(D \cup \{y_1, y_2\}\big) + f\big(D\big) 
                                                          \label{eq:submodularity}
\end{align}
with $D\subset\Omega$ and $y_1,y_2\in\Omega\setminus D$.
\end{definition}

\begin{definition}
  For a random variable $\theta$ and a set of random variables $Y$,
  \begin{align}
    H(\theta|Y) = \sum_{Y} p(Y) H[p(\theta|Y)]
  \end{align}
is the conditional entropy.
\end{definition}

\begin{lemma}
  $f: 2^\Omega \rightarrow \mathbb{R}$ with $f(Y) = H(\theta|Y)$ is not submodular.
\end{lemma}

\begin{proof}

We proof by contradiction, giving an example that violates
\begin{align}
  H(\theta|\emptyset \cup \{y_1\}) + H(\theta|\emptyset \cup \{y_2\}) &\geq H(\theta|\emptyset \cup \{y_1,y_2\}) +
  H(\theta|\emptyset)~.
                                      \label{eq:submodularity_conditional_entropy}
\end{align}

Let $\theta$ be a binary random variable, and $y_1$ and $y_2$ be identically
distributed binary random variables with
\begin{align}
  &&&&\text{prior:}&&p(\theta) &= \left(\begin{array}{l}0.5\\0.5\end{array}\right)&&&&\\
  &&&&\text{likelihood:}&&p(y|\theta) &= \left(\begin{array}{ll}0.1&0.9\\0.9&0.1\end{array}\right)\\
  &&&&\text{marginal:}&&p(y) &= \sum_{\theta} p(y|\theta)p(\theta) = \left(\begin{array}{l}0.5\\0.5\end{array}\right)\\
  &&&&\text{posterior:}&&p(\theta|y) &= \frac{p(y|\theta)p(\theta)}{p(y)}
  = \left(\begin{array}{ll}0.1&0.9\\0.9&0.1\end{array}\right),
\end{align}
so that
\begin{align}
  H(\theta|y_1) &= H(\theta|y_2) = H(\theta|y) = - \sum_{y} p(y) \sum_{\theta} p(\theta|y) \log p(\theta|y) = 0.325, \\
  H(\theta|\emptyset) &= H(\theta) = \sum_{\theta} p(\theta) \log p(\theta) = 0.693
\end{align}
and
\begin{align}
  H(\theta|y_1) + H(\theta|y_2) = 2 \cdot 0.325
  < 0.693 = H(\theta) \leq H(\theta|y_1,y_2) + H(\theta)~,
\end{align}
which contradicts \eqref{eq:submodularity_conditional_entropy}.
\end{proof}

\end{document}